\documentclass[letterpaper]{article}
\usepackage{proceed2e}
\usepackage[margin=1in]{geometry}

\usepackage[round]{natbib}
\usepackage[resetlabels,labeled]{multibib}
\newcites{Supp}{References}

\usepackage{times}
\usepackage{xcolor}
\usepackage{soul}
\usepackage[utf8]{inputenc}
\usepackage[small]{caption}

\usepackage[ruled]{algorithm2e}
\SetKwInOut{Input}{input}
\SetKwInOut{Output}{output}
\SetKw{Throw}{throw}

\usepackage{enumitem}
\usepackage{mathtools}
\usepackage{tikz}
\usetikzlibrary{arrows}
\usepackage{amsthm}
\usepackage{amssymb}
\usepackage{subcaption}
\usepackage{verbatim}

\newtheorem{defn}{Definition}
\newtheorem{theorem}{Theorem}
\newtheorem{lemma}{Lemma}
\newtheorem{proposition}{Proposition}
\newtheorem{corollary}{Corollary}
\newenvironment{psketch}{%
  \proof}{\endproof}

\DeclareMathAlphabet\mathbfcal{OMS}{cmsy}{b}{n}

\makeatletter
\newcommand{\removelatexerror}{\let\@latex@error\@gobble}
\newcommand{\xdashleftrightarrow}[2][]{\ext@arrow 3359\leftrightarrowfill@@{#1}{#2}}
\def\leftrightarrowfill@@{\arrowfill@@\leftarrow\relbar\rightarrow}
\def\arrowfill@@#1#2#3#4{%
  $\m@th\thickmuskip0mu\medmuskip\thickmuskip\thinmuskip\thickmuskip
   \relax#4#1
   \xleaders\hbox{$#4#2$}\hfill
   #3$%
}

\makeatother

\title{Causal Identification under Markov Equivalence}

\author{ {\bf Amin Jaber} \\
Computer Science Department \\
Purdue University, IN, USA \\
jaber0@purdue.edu \\
\And
{\bf Jiji Zhang}  \\
Philosophy Department \\
Lingnan University, NT, HK \\
jijizhang@ln.edu.hk \\
\And
{\bf Elias Bareinboim} \\
Computer Science Department \\
Purdue University, IN, USA\\
eb@purdue.edu \\
}

\begin{document}

\maketitle

\begin{abstract}
Assessing the magnitude of cause-and-effect relations is one of the central challenges found throughout the empirical sciences.
The problem of identification of causal effects is concerned with determining whether a causal effect can be computed from a combination of observational data and substantive knowledge about the domain under investigation, which is formally expressed in the form of a causal graph. In many practical settings, however, the knowledge available for the researcher is not strong enough so as to specify a unique causal graph.
Another line of investigation attempts to use observational data to learn a qualitative description of the domain called a Markov equivalence class, which is the  collection of causal graphs that share the same set of observed features.  In this paper, we marry both approaches and study the problem of causal identification from an equivalence class, represented by a partial ancestral graph (PAG).
We start by deriving a set of graphical properties of PAGs that are carried over to its induced subgraphs. We then develop an algorithm to compute the effect of an arbitrary set of variables on an arbitrary outcome set. We show that the algorithm is strictly more powerful than the current state of the art found in the literature.
\end{abstract}

\section{INTRODUCTION}
\label{sec:intro}
Science is about explaining the mechanisms underlying a phenomenon that is being investigated. One of the marks imprinted by these mechanisms in reality is cause and effect relationships.
Systematically discovering the existence, and magnitude, of causal relations constitutes, therefore, a central task in scientific domains.
The value of inferring causal relationships  is also tremendous in other, more practical domains, including, for example, engineering and business, where it is often crucial to understand how to bring about a specific change when a constrained amount of controllability is in place. If our goal is to build AI systems that can act and learn autonomously, formalizing the principles behind causal inference, so that these systems can leverage them, is a fundamental requirement \citep{pearl2018}. 
 
One prominent approach to infer causal relations leverages a combination of substantive knowledge about the domain under investigation, usually encoded in the form of a causal graph, with observational (non-experimental) data  \citep{pearl:2k,spirtes2001causation,bareinboim:pea:16-r450}.
A sample causal graph is shown in Fig.~\ref{fig:introdag} such that the nodes represent variables, directed edges represent direct causal relation from tails to heads, and bi-directed arcs represent the presence of unobserved (latent) variables that generate a spurious association between the variables, also known as \textit{confounding bias}~\citep{pearl:93c}.
The task of determining whether an interventional (experimental) distribution can be computed from a combination of observational and experimental data together with the causal graph is known as the problem of identification of causal effects (identification, for short). 
For instance, a possible task in this case is to identify the effect of $do(X\!\!=\!\!x)$ on $V_4\!\!=\!\!v_4$, i.e. $P_x(v_4)$, given the causal graph in Fig.~\ref{fig:introdag} and data from the observational distribution $P(x,v_1, ..., v_4)$.

The problem of identification has been extensively studied in the literature, and a number of criteria have been established~\citep{pearl:93c,galles:pea95,kuroki1999identifiability,tian2002general,Huang:2006:PCI:3020419.3020446,shpitser2006identification,bareinboim2012causal}, which include the celebrated back-door criterion and the do-calculus \citep{pearl1995causal}.
Despite their power, these techniques require 
a fully specified causal graph, which 
is not always available 
in practical settings.

Another line of investigation attempts to learn a qualitative description of the system, which in the ideal case would lead to the ``true'' data-generating model, the blueprint underlying the phenomenon being investigated. These efforts could certainly be deemed more ``data-driven'' and aligned with the 
zeitgeist in machine learning.
In practice, however, it is 
common that only an equivalence class of causal models can be consistently inferred from observational data~\citep{verma1993graphical,spirtes2001causation,zhang2008completeness}.
One useful characterization of such an equivalence class comes under the rubric of \textit{partial ancestral graphs (PAGs)}, which will be critical to our work.
Fig.~\ref{fig:introex} shows the PAG (right) that can be inferred from observational data that is consistent with the true causal model (left).
The directed edges in a PAG signify ancestral relations (not necessarily direct) and circle marks stand for structural uncertainty.

In this paper, we analyze the marriage of these two lines of investigation, where the structural invariance learned in the equivalence class will be used as input to identify the strength of 
causal effect relationships, if possible.
Identification from an equivalence class is considerably more challenging than from a single diagram due to the structural uncertainty regarding both the direct causal relations among the variables and the presence of latent variables that confounds causal relations between observed variables.
Still, there is a growing interest in identifiability results in this setting \citep{marloes10}.
\citet{zhang2007generalized} extended the do-calculus to PAGs.
In practice, however, it is in general 
computationally hard to decide whether there exists (and, if so, find) a sequence of applications of the rules of the generalized calculus to identify the interventional distribution.
\citet{perkovic2015complete} generalized the back-door criterion to PAGs, and provided a sound and complete algorithm to find a back-door admissible set, should such a set exist.
However, in practice, the back-door criterion is not as powerful as the do-calculus, since no adjustment set exists for many identifiable causal effects.
\citet{jaber18} generalized the work of \citep{tian2002general} and devised a graphical criterion to identify causal effects with singleton interventions in PAGs.
\footnote{Another possible approach is based on  SAT (boolean constraint satisfaction) solvers  \citep{hyttinen2015calculus}. Given its somewhat distinct nature, a closer comparison lies outside the scope of this paper. We note, however, that an open research direction would be to translate our systematic approach into logical rules so as to help improving the solver's scalability.}

Building on this work, we develop here a decomposition strategy akin to the one introduced in~\citep{tian2002studies} to identify causal effects given a PAG.
Our proposed approach is computationally more attractive than the do-calculus as it provides a systematic procedure to identify a causal effect, if identifiable.
It is also more powerful than the generalized adjustment criterion, as we show later.
More specifically, our main contributions are: 
\begin{enumerate}
\item We study some critical properties of PAGs and show that they also hold in induced subgraphs of a PAG over an arbitrary subset of nodes. We further study Tian's c-component decomposition and relax it to PAGs (when only partial knowledge about the ancestral relations and c-components is available).
\item We formulate a systematic procedure to compute the effect of an arbitrary set of intervention variables on an arbitrary outcome set from a PAG and observational data. We show that this algorithm is strictly more powerful than the adjustment criterion. 
\end{enumerate}

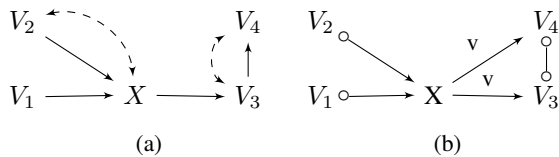
\begin{figure}[t]
\centering
\begin{subfigure}{0.5\columnwidth}
\begin{tikzpicture}
\tikzset{vertex/.style = {shape=circle,draw,minimum size=1.5em}}
\tikzset{edge/.style = {->,> = latex'}}

\node[below] (v1) at (0,0) {$V_1$};
\node[below] (v2) at (0,1) {$V_2$};
\node[below] (x) at (1.5,0) {$X$};
\node[below] (v3) at (3,0) {$V_3$};
\node[below] (v4) at (3,1) {$V_4$};

\draw[edge] (v1) to (x);
\draw[edge] (v2) to (x);
\draw[latex'-latex',dashed] (v2) to [bend left=50]  (x);
\draw[edge] (x) to (v3);
\draw[edge] (v3) to (v4);
\draw[latex'-latex',dashed] (v3) to [bend left=60] (v4);

\end{tikzpicture}
\caption{\label{fig:introdag}}
\end{subfigure}%
\begin{subfigure}{0.5\columnwidth}
\begin{tikzpicture}
\tikzset{vertex/.style = {shape=circle,draw,minimum size=1.5em}}
\tikzset{edge/.style = {->,> = latex'}}

\node[below] (v1) at (0,0) {$V_1$};
\node[below] (v2) at (0,1) {$V_2$};
\node[below] (x) at (1.5,0) {X};
\node[below] (v3) at (3,0) {$V_3$};
\node[below] (v4) at (3,1) {$V_4$};

\draw[shorten <= 2pt,edge] (v1) to node[at start]{$\circ$} (x);
\draw[shorten <= 2pt,edge] (v2) to node[at start]{$\circ$} (x);
\draw[edge] (x) to node[auto=left] {\small v}  (v3);
\draw[edge] (x) to node[auto=left] {\small v} (v4);
\draw[shorten <= 2pt,shorten >= 2pt] (v3) to node[at start]{$\circ$} node[at end]{$\circ$} (v4);

\end{tikzpicture}
\caption{\label{fig:intropag}}
\end{subfigure}
\caption{\label{fig:introex}A causal model (left) and the inferred PAG (right).}
\end{figure}

\section{PRELIMINARIES}
\label{sec:prelim}
In this section, we introduce the basic notation and machinery used throughout the paper.
Bold capital letters denote sets of variables, while bold lowercase letters stand for particular assignments to those variables.

\textbf{Structural Causal Models. }
We use the language of Structural Causal Models (SCM) \citep[pp. 204-207]{pearl:2k} as our basic semantic framework.
Formally, an SCM $M$ is a 4-tuple $\langle U,V,F,P(u)\rangle$, where $U$ is a set of exogenous (latent) variables and $V$ is a set of endogenous (measured) variables.
$F$ represents a collection of functions $F=\{f_i\}$ such that each endogenous variable $V_i\in V$ is determined by a function $f_i\in F$, where $f_i$ is a mapping from the respective domain of $U_i \cup Pa_i$ to $V_i$, $U_i\subseteq U$, $Pa_i\subseteq V\setminus V_i$. 
The uncertainty is encoded through a probability distribution over the exogenous variables, $P(u)$. 
A causal diagram associated with an SCM encodes the structural relations among 
$V\cup U$, in which an arrow is drawn from each member of $U_i \cup Pa_i$ to $V_i$.
We constraint our results to recursive systems, which means that the corresponding  diagram will be acyclic.
The marginal distribution over the endogenous variables $P(v)$ is called observational, and factorizes according to the causal diagram, i.e.:
\begin{align*}
P(v)=\sum_{u}\prod_i P(v_i|pa_i,u_i) P(u)
\end{align*}

Within the structural semantics, performing an action $X\!\! =\!\! x$ is represented through the do-operator, \textit{do}($X\!=\!x$), which encodes the operation of replacing the original equation for $X$ by the constant $x$ and induces a submodel $M_x$.
The resulting distribution is denoted by $P_x$, which is the main target for identification in this paper.
For details on structural models, we refer readers to \citep{pearl:2k}. 

\textbf{Ancestral Graphs. }
We now introduce a graphical representation of equivalence classes of causal diagrams.
A \textit{mixed} graph can contain directed ($\rightarrow$) and bi-directed edges ($\leftrightarrow$).
$A$ is a \textit{spouse} of $B$ if $A\leftrightarrow B$ is present.
An \textit{almost directed cycle} happens when $A$ is both a spouse and an ancestor of $B$.
An \textit{inducing path relative to} $\mathbf{L}$ is a path on which every node $V\notin \mathbf{L}$ 
(except for the endpoints) is a collider on the path (i.e., both edges incident to $V$ are into $V$) and every collider is an ancestor of an endpoint of the path.
A mixed graph is \textit{ancestral} if it doesn't contain a directed or almost directed cycle.
It is \textit{maximal} if there is no inducing path (relative to the empty set) between any two non-adjacent nodes.
A \textit{Maximal Ancestral Graph} (MAG) is a graph that is both ancestral and maximal.
 MAG models are closed under marginalization ~\citep{richardson2002ancestral}.

In general, a causal MAG represents a set of causal models with the same set of observed variables that entail the same independence and ancestral relations among the observed variables.
Different MAGs may be Markov equivalent in that they entail the exact same independence model.
A partial ancestral graph (PAG) represents an equivalence class of MAGs $[\mathcal{M}]$, which shares the same adjacencies as every MAG in $[\mathcal{M}]$ and displays all and only the invariant edge marks.

\begin{defn}[PAG]
Let $[\mathcal{M}]$ be the Markov equivalence class of an arbitrary MAG $\mathcal{M}$.
The PAG for $[\mathcal{M}]$, $\mathcal{P}$, is a partial mixed graph such that:
\begin{enumerate}[label=\roman*.]
\item $\mathcal{P}$ has the same adjacencies as $\mathcal{M}$ (and hence any member of $[\mathcal{M}]$) does.
\item An arrowhead is in $\mathcal{P}$ iff shared by all MAGs in $[\mathcal{M}]$.
\item A tail is in $\mathcal{P}$ iff shared by all MAGs in $[\mathcal{M}]$.
\item A mark that is neither an arrowhead nor a tail is recorded as a circle.
\end{enumerate}
\end{defn}

A PAG is learnable from the conditional independence and dependence relations among the observed variables and the FCI algorithm is a standard method to learn such an object~\citep{zhang2008completeness}.
In short, a PAG represents an equivalence class of causal models with the same observed variables and independence model.

\textbf{Graphical Notions. }
Given a DAG, MAG, or PAG, a path between $X$ and $Y$ is \textit{potentially directed (causal)} from $X$ to $Y$ if there is no arrowhead on the path pointing towards $X$.
$Y$ is called a \textit{possible descendant} of $X$ and $X$ a \textit{possible ancestor} of $Y$ if there is a potentially directed path from $X$ to $Y$.
A set $\mathbf{A}$ is (\textit{descendant}) \textit{ancestral} if no node outside $\mathbf{A}$ is a possible (descendant) ancestor of any node in $\mathbf{A}$.
$Y$ is called a \textit{possible child} of $X$, i.e. $Y\in\texttt{Ch}(X)$, and $X$ a \textit{possible parent} of $Y$, i.e. $X\in\texttt{Pa}(Y)$, if they are adjacent and the edge is not into $X$.
For a set of nodes $\mathbf{X}$, we have $\texttt{Pa}(\mathbf{X})=\cup_{X\in\mathbf{X}}\texttt{Pa}(X)$ and $\texttt{Ch}(\mathbf{X})=\cup_{X\in\mathbf{X}}\texttt{Ch}(X)$.
Given two sets of nodes $\mathbf{X}$ and $\mathbf{Y}$, a path between them is called \textit{proper} if one of the endpoints is in $\mathbf{X}$ and the other is in $\mathbf{Y}$, and no other node on the path is in $\mathbf{X}$ or $\mathbf{Y}$.
For convenience, we use an 
asterisk (*) to denote any possible mark of a PAG ($\circ,>,\--$) or a MAG ($>,\--$).
If the edge marks on a path between $X$ and $Y$ are all circles, we call the path a \textit{circle path}.





A directed edge $X\rightarrow Y$ in a MAG or PAG is \textit{visible} if there exists no DAG $\mathcal{D}(\mathbf{V},\mathbf{L})$ in the corresponding equivalence class where there is an inducing path between $X$ and $Y$ that is into $X$ relative to $\mathbf{L}$.
This implies that a visible edge is not confounded ($X\leftarrow U_i \rightarrow Y$ doesn't exist).
Which directed edges are visible is easily decidable by a graphical condition \citep{zhang2008causal}, so we simply mark visible edges by ${\small v}$.
For brevity, we refer to any edge that is not a visible directed edge as \textit{invisible}.

\textbf{Identification Given a Causal DAG. }
\citet{tian2002general} presented an identification algorithm based on a decomposition strategy of the DAG into a set of so-called \textit{c-components} (confounded components).

\begin{defn}[C-Component]
In a causal DAG, two observed variables are said to be in the same c-component if and only if they are connected by a bi-directed path, i.e. a path composed solely of such bi-directed treks as $V_i\leftarrow U_{ij} \rightarrow V_j$, where $U_{ij}$ is an exogenous variable.
\end{defn}

For convenience, we often refer to a bi-directed trek like $V_i\leftarrow\!U_{ij}\!\rightarrow V_j$ as a bi-directed edge between $V_i$ and $V_j$ (and $U_{ij}$ is often left implicit). 
For any set $\mathbf{C}\subseteq\mathbf{V}$, we define the quantity $Q[\mathbf{C}]$ 
to denote the post-intervention distribution of $\mathbf{C}$ under an intervention on $\mathbf{V}\setminus\mathbf{C}$:
\begin{align*}
Q[\mathbf{C}]=P_{\mathbf{v}\backslash\mathbf{c}}(\mathbf{c})=\sum_u \prod_{\{i|V_i\in\mathbf{C}\}} P(v_i|pa_i,u_i)P(u)
\end{align*}

\vspace{-1em}
The significance of c-components and their decomposition is evident from  \cite[Lemmas~10, 11]{tian2002studies}, which are the basis of Tian's identification algorithm.

\vspace{-0.1in}

\section{REVISIT IDENTIFICATION IN DAGS}
\label{sec:dagrevisited}
We revisit the identification results in DAGs, focusing on Tian's algorithm \citep{tian2002studies}. Our goal here is to have an amenable algorithm that allows the incorporation of the structural uncertainties arising 
in the equivalence class.

Let $\mathcal{D}_\mathbf{A}$ denote the (induced) subgraph of a DAG $\mathcal{D}(\mathbf{V},\mathbf{L})$ over $\mathbf{A}\subseteq\mathbf{V}$ and the latent parents of $\mathbf{A}$ (i.e. $\texttt{Pa}(\mathbf{A})\cap\mathbf{L}$).
The original algorithm (Alg.~5 in \citep{tian2002studies}) alternately applies Lemmas 10 and 11 in \citep{tian2002studies} until a solution is derived or a failure condition is triggered. We rewrite this algorithm with a more local, atomic criterion based on the following results.

\begin{defn}[Composite C-Component]
Given a DAG that decomposes into c-components $S_1,\dots,S_k$, $k\geq 1$, 
a composite c-component is the union of one or more of these c-components.
\label{def:composite-ccomp}
\end{defn}

\begin{lemma}
Given a DAG $\mathcal{D}(\mathbf{V},\mathbf{L})$, $\mathbf{X}\subset\mathbf{T}\subseteq\mathbf{V}$, and $P_{\mathbf{v}\setminus\mathbf{t}}$ the interventional distribution of $\mathbf{V}\setminus\mathbf{T}$ on $\mathbf{T}$.
Let $S^\mathbf{X}$ denote a composite c-component containing $\mathbf{X}$ in $\mathcal{D}_\mathbf{T}$.
If $\mathbf{X}$ is a descendant set in $\mathcal{D}_{S^\mathbf{X}}$, then $Q[\mathbf{T}\setminus \mathbf{X}]$ is identifiable and given by
\begin{align}
Q[\mathbf{T}\setminus \mathbf{X}] = \frac{P_{\mathbf{v}\setminus\mathbf{t}}}{Q[S^\mathbf{X}]}\times \sum_\mathbf{x} Q[S^\mathbf{X}]
\label{eq:subsetiddag}
\end{align}
\label{lem:subsetiddag}
\end{lemma} 
\begin{proof}
 By \citep[Lemma 11]{tian2002studies}, $Q[\mathbf{T}]$ decomposes as follows.
\begin{align*}
Q[\mathbf{T}] = Q[\mathbf{T}\setminus S^\mathbf{X}]\times Q[S^\mathbf{X}] = \frac{Q[\mathbf{T}]}{Q[S^\mathbf{X}]}\times Q[S^\mathbf{X}]
\end{align*}

$Q[S^\mathbf{X}]$ is computable from $P_{\mathbf{v}\setminus\mathbf{t}}$ using Lemma 11 in \citep{tian2002studies}, and $Q[S^\mathbf{X}\setminus\mathbf{X}]$ is computable from $Q[S^\mathbf{X}]$ using \citep[Lemma 10]{tian2002studies} as $\mathbf{X}$ is a descendant set in $\mathcal{D}_{S^\mathbf{X}}$.
Therefore,
\begin{align*}
Q[\mathbf{T}\setminus\mathbf{X}]	&= \frac{P_{\mathbf{v}\setminus\mathbf{t}}}{Q[S^\mathbf{X}]}\cdot Q[S^\mathbf{X}\setminus\mathbf{X}] = \frac{P_{\mathbf{v}\setminus\mathbf{t}}}{Q[S^\mathbf{X}]}\cdot \sum_\mathbf{x} Q[S^\mathbf{X}]
\end{align*}
\end{proof}

The next result follows directly when $\mathbf{X}$ is a singleton.

\begin{corollary}
Given a DAG $\mathcal{D}(\mathbf{V},\mathbf{L})$, $X\in\mathbf{T}\subseteq\mathbf{V}$, and $P_{\mathbf{v}\setminus\mathbf{t}}$.
If $X$ is not in the same c-component with a child in $\mathcal{D}_\mathbf{T}$, then $Q[\mathbf{T}\setminus \{X\}]$ is identifiable and given by
\begin{align}
Q[\mathbf{T}\setminus \{X\}] = \frac{P_{\mathbf{v}\setminus\mathbf{t}}}{Q[S^X]}\times \sum_x Q[S^X]
\end{align}
where $S^X$ is the c-component of $X$ in $\mathcal{D}_\mathbf{T}$.
\label{cor:singleiddag}
\end{corollary}

The significance of Corol.~\ref{cor:singleiddag} stems from the fact that it can be used to rewrite the identification algorithm in a step-wise fashion, which is shown in
Algorithm~\ref{alg:iddag}. The same is equivalent to the original algorithm since neither one of Lemmas 10 nor 11 in \citep{tian2002studies} is applicable whenever Corol.~\ref{cor:singleiddag} is not applicable, which is shown by Lemmas~\ref{lem:ancs-single} and \ref{lem:decomp-single}.
This result may not be surprising since Corol.~\ref{cor:singleiddag} follows from the application of these lemmas.

\begin{algorithm}[t]
\Input{two disjoint sets $\mathbf{X},\mathbf{Y}\subset\mathbf{V}$}
\Output{Expression for $P_\mathbf{x}(\mathbf{y})$ or \texttt{FAIL}}
\begin{enumerate}
\item Let $\mathbf{D}=\texttt{An}(\mathbf{Y})_{\mathcal{G}_{\mathbf{V}\setminus\mathbf{X}}}$
\item Let the c-components of $\mathcal{G}_{\mathbf{D}}$ be $\mathbf{D}_{i}$, $i=1,\dots,k$
\item $P_\mathbf{x}(\mathbf{y})=\sum_{\mathbf{d}\setminus\mathbf{y}}\prod_i \texttt{Identify}(\mathbf{D}_{i},\mathbf{V},P)$
\end{enumerate}
\SetKwFunction{FIdentify}{Identify}
  \SetKwProg{Fn}{Function}{:}{}
  \Fn{\FIdentify{$\mathbf{C}$, $\mathbf{T}$, $Q=Q[\mathbf{T}]$}}{
\If{$\mathbf{C}=\mathbf{T}$}{
    \Return $Q[\mathbf{T}]$\;
}

\relsize{-1}{\tcc{Let $S^B$ be the c-component of $\{B\}$ in $\mathcal{G}_\mathbf{T}$}}\normalsize
\uIf{$\exists B\in\mathbf{T}\setminus\mathbf{C}$ such that $S^B\cap \texttt{Ch}(B)=\emptyset$}{
	Compute $Q[\mathbf{T}\setminus\{B\}]$ from $Q$; \relsize{-1}{\tcp{Corollary~\ref{cor:singleiddag}}}\normalsize
	\Return $\texttt{Identify}(\mathbf{C},\mathbf{T}\setminus\{B\},Q[\mathbf{T}\setminus\{B\}])$\;
}
\Else{
	\Throw \texttt{FAIL}\;
}
  }
\caption{\label{alg:iddag}$\mathbf{ID}(\mathbf{x},\mathbf{y})$ given DAG $\mathcal{G}$}
\end{algorithm}

\begin{lemma}
Given a DAG $\mathcal{D}(\mathbf{V},\mathbf{L})$, 
$\mathbf{C}\subset\mathbf{T}\subseteq\mathbf{V}$.
If $\mathbf{A}=\texttt{An}(\mathbf{C})_{\mathcal{D}_\mathbf{T}}\neq \mathbf{T}$, then there exist some node $X\in\mathbf{T}\setminus\mathbf{A}$ such that $X$ is not in the same c-component with any child in $\mathcal{D}_\mathbf{T}$.
\label{lem:ancs-single}
\end{lemma}

\begin{proof}
If $\mathbf{A}\neq \mathbf{T}$, then $\mathbf{T}\setminus\mathbf{A}$ is a non-empty set where none of the nodes is an ancestor of $\mathbf{A}$.
Since the graph is acyclic, then at least one node of $\mathbf{T}\setminus\mathbf{A}$  is with no children.
Hence, the above conclusion follows.
\end{proof}

\begin{lemma}
Given a DAG $\mathcal{D}(\mathbf{V},\mathbf{L})$, 
$\mathbf{C}\subset\mathbf{T}\subseteq\mathbf{V}$,
and assume $\mathcal{D}_\mathbf{C}$ is a single c-component.
If $\mathcal{D}_\mathbf{T}$ partitions into c-components $\mathbf{S}_1\dots\mathbf{S}_k$, where $k>1$, then there exists some node $X\in\mathbf{S}_i$ such that $\mathbf{C}\not\subseteq\mathbf{S}_i$ and $X$ is not in the same c-component with any child in $\mathcal{D}_\mathbf{T}$.
\label{lem:decomp-single}
\end{lemma}

\begin{proof}
Subgraph $\mathcal{D}_{\mathbf{S}_i}$ is acyclic, so there must exist some node ($X$) that doesn't have any children in $\mathcal{D}_{\mathbf{S}_i}$.
Since $\mathbf{S}_i$ is one of the c-components in $\mathcal{D}_\mathbf{T}$, then $X$ is not in the same c-component with any of its children in $\mathcal{D}_\mathbf{T}$.
\end{proof}

The revised algorithm requires checking an atomic criterion at every instance of the recursive routine $\texttt{Identify}$.
This might not be crucial when the precise causal diagram is known and induced subgraphs preserve a complete graphical characterization of the c-components and the ancestral relations between the nodes.
The latter, unfortunately, doesn't hold when the model is an equivalence class represented by a PAG.\footnote{
We thank a reviewer for bringing to our attention a similar formulation of Alg. 1  ~\citep[Thm.~60]{richardson2017nested}. 
}
\begin{figure*}[t]
\centering
\begin{subfigure}{0.5\columnwidth}
\begin{tikzpicture}
\tikzset{vertex/.style = {shape=circle,draw,minimum size=1.5em}}
\tikzset{edge/.style = {->,> = latex'}}

\node[below] (v1) at (0,0) {$V_1$};
\node[below] (v2) at (0,1) {$V_2$};
\node[below] (x) at (1.5,0) {X};
\node[below] (v4) at (3,1) {$V_4$};

\draw[shorten <= 2pt,edge] (v1) to node[at start]{$\circ$} (x);
\draw[shorten <= 2pt,edge] (v2) to node[at start]{$\circ$} (x);
\draw[edge] (x) to node[auto=left] {\small v} (v4);

\end{tikzpicture}
\caption{\label{fig:subpag}Subgraph of Fig.~\ref{fig:intropag}.}
\end{subfigure}%
\begin{subfigure}{0.5\columnwidth}
\begin{tikzpicture}
\tikzset{vertex/.style = {shape=circle,draw,minimum size=1.5em}}
\tikzset{edge/.style = {->,> = latex'}}

\node[below] (v1) at (0,0) {$V_1$};
\node[below] (v2) at (0,1) {$V_2$};
\node[below] (x) at (1.5,0) {$X$};
\node[below] (v4) at (3,1) {$V_4$};

\draw[edge] (v1) to (x);
\draw[edge] (v2) to (x);
\draw[latex'-latex',dashed] (v2) to [bend left=50]  (x);

\end{tikzpicture}
\caption{\label{fig:subdag1}Subgraph of Fig.~\ref{fig:introdag}.}
\end{subfigure}%
\begin{subfigure}{0.5\columnwidth}
\begin{tikzpicture}
\tikzset{vertex/.style = {shape=circle,draw,minimum size=1.5em}}
\tikzset{edge/.style = {->,> = latex'}}

\node[below] (v1) at (0,0) {$V_1$};
\node[below] (v2) at (0,1) {$V_2$};
\node[below] (x) at (1.5,0) {$X$};
\node[below] (v3) at (3,0) {$V_3$};
\node[below] (v4) at (3,1) {$V_4$};

\draw[latex'-latex',dashed] (v1) to (x);
\draw[latex'-latex',dashed] (v2) to (x);
\draw[edge] (x) to (v3);
\draw[edge] (x) to (v4);
\draw[edge] (v3) to (v4);
\draw[latex'-latex',dashed] (v3) to [bend right=60] (v4);

\end{tikzpicture}
\caption{\label{fig:eqdag2}Equivalent DAG to Fig.~\ref{fig:introdag}.}
\end{subfigure}%
\begin{subfigure}{0.5\columnwidth}
\begin{tikzpicture}
\tikzset{vertex/.style = {shape=circle,draw,minimum size=1.5em}}
\tikzset{edge/.style = {->,> = latex'}}

\node[below] (v1) at (0,0) {$V_1$};
\node[below] (v2) at (0,1) {$V_2$};
\node[below] (x) at (1.5,0) {$X$};
\node[below] (v4) at (3,1) {$V_4$};

\draw[latex'-latex',dashed] (v1) to (x);
\draw[latex'-latex',dashed] (v2) to (x);
\draw[edge] (x) to (v4);

\end{tikzpicture}
\caption{\label{fig:subdag2}Subgraph of Fig.~\ref{fig:eqdag2}}
\end{subfigure}
\caption{\label{fig:subpagex}Example for properties discussed in Section~\ref{sec:subpag}}
\end{figure*}
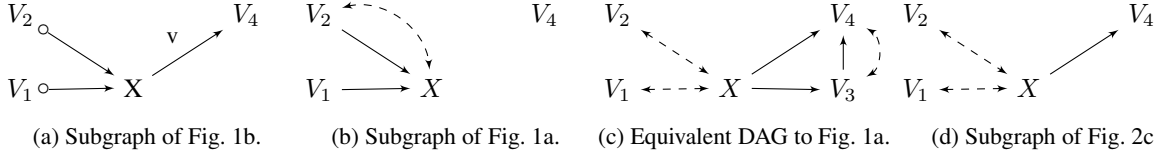

\section{PAG-SUBGRAPH PROPERTIES}
\label{sec:subpag}
Evidently, induced subgraphs of the original causal model play a critical role in identification (cf Alg.~\ref{alg:iddag}). It is natural to expect that in the generalized setting we study here, induced subgraphs of the given PAG will also play an important role. An immediate challenge, however, is that a subgraph of a PAG $\mathcal{P}$ over $\mathbf{V}$ induced by $\mathbf{A}\subseteq\mathbf{V}$ is, in general, not a PAG that represents a full Markov equivalence class. In particular, if $\mathcal{D}(\mathbf{V},\mathbf{L})$ is a DAG in the equivalence class represented by $\mathcal{P}$, $\mathcal{P}_\mathbf{A}$ is in general not the PAG that represents the equivalence class of $\mathcal{D}_\mathbf{A}$. 
To witness, let $\mathcal{D}$ and $\mathcal{P}$ denote the DAG and the corresponding PAG in Figure~\ref{fig:introex}, respectively, and let $\mathbf{A}=\{V_1,V_2,X,V_4\}$.
The induced subgraph of $\mathcal{P}$ over $\mathbf{A}$ (Fig.~\ref{fig:subpag}) does not represent the equivalence class of the corresponding induced subgraph of $\mathcal{D}$ (Fig.~\ref{fig:subdag1}).
Despite this subtlety, we establish a few facts below showing that for any $\mathbf{A}\subseteq \mathbf{V}$ 
and any DAG $\mathcal{D}$ in the equivalence class represented by $\mathcal{P}$, some information about $\mathcal{D}_\mathbf{A}$, which is particularly relevant to identification, can be read off from $\mathcal{P}_\mathbf{A}$.
 

\begin{proposition}
Let $\mathcal{P}$ be a PAG over $\mathbf{V}$, and $\mathcal{D}(\mathbf{V},\mathbf{L})$ be any DAG in the equivalence class represented by $\mathcal{P}$.
Let $X\neq Y$ be two nodes in $\mathbf{A}\subseteq\mathbf{V}$.
If $X$ is an ancestor of $Y$ in $\mathcal{D}_\mathbf{A}$, then $X$ is a possible ancestor of $Y$ in $\mathcal{P}_\mathbf{A}$.
\label{prop:ancsubpag}
\end{proposition}

\begin{proof}
If $X$ is an ancestor of $Y$ in $\mathcal{D}_\mathbf{A}$, then there is a path $p$ in $\mathcal{D}_\mathbf{A}$ composed of nodes $\langle X=V_0,\dots,Y=V_m\rangle$, $m\geq 1$ such that $V_i\in\mathbf{A}$ and $V_i\rightarrow V_{i+1}$, $0\leq i<m$.
Path $p$ is obviously also present in $\mathcal{D}$, and consequently the corresponding MAG $\mathcal{M}$.
Hence, $p$ corresponds to a possibly directed path in $\mathcal{P}$.
Since all the nodes along $p$ are in $\mathbf{A}$, then $p$ is present in $\mathcal{P}_\mathbf{A}$ and so $X$ is a possible ancestor of $Y$ in $\mathcal{P}_\mathbf{A}$.
\end{proof}

This simple proposition guarantees that possible-ancestral relationship in $\mathcal{P}_\mathbf{A}$ subsumes ancestral relationship in $\mathcal{D}_\mathbf{A}$ for every $\mathcal{D}$ in the class represented by $\mathcal{P}$. This is illustrated by $\mathcal{D}_\mathbf{A}$ and $\mathcal{P}_\mathbf{A}$ in Figures \ref{fig:subpag} and \ref{fig:subdag1}.

Given an induced subgraph of a PAG, $\mathcal{P}_\mathbf{A}$, a directed edge $X\rightarrow Y$ in $\mathcal{P}_\mathbf{A}$ is said to be {\em visible} if for every DAG $\mathcal{D}$ in the class represented by $\mathcal{P}$, there is no inducing path in $\mathcal{D}_\mathbf{A}$ between $X$ and $Y$ relative to the latent nodes in $\mathcal{D}_\mathbf{A}$ that is into $X$.

\begin{lemma}
Let $\mathcal{P}$ be a PAG over $\mathbf{V}$, and $\mathcal{P}_\mathbf{A}$ be an induced subgraph of $\mathcal{P}$ over $\mathbf{A}\subseteq \mathbf{V}$. For every $X\rightarrow Y$ in $\mathcal{P}_\mathbf{A}$, if it is visible in $\mathcal{P}$, then it remains visible in $\mathcal{P}_\mathbf{A}$.
\label{lem:subpagvis}
\end{lemma}

\begin{proof}
Let $\mathcal{D}(\mathbf{V},\mathbf{L})$ be any causal model in the equivalence class represented by $\mathcal{P}$, and let $X\rightarrow Y$ be a visible edge in $\mathcal{P}$, $X,Y\in\mathbf{A}$.
Then, there is no inducing path between $X$ and $Y$ relative to $\mathbf{L}$ that is into $X$ in $\mathcal{D}$.
It follows that no such inducing path (relative to the latent nodes in $\mathcal{D}_\mathbf{A}$) exists in the subgraph $\mathcal{D}_\mathbf{A}$.
\end{proof}

Visibility is relevant for identification because it implies absence of confounding, which is the major obstacle to identification. Lemma~\ref{lem:subpagvis} shows that an edge in an induced subgraph that is visible in the original PAG also implies absence of confounding in the induced subgraphs.
Interestingly, note that a directed edge $X\rightarrow Y$ in $\mathcal{P}_\mathbf{A}$, visible or not, does not imply that $X$ is an ancestor of $Y$ in $\mathcal{D}_\mathbf{A}$ for every $\mathcal{D}$ in the class represented by $\mathcal{P}$.
For example, $X$ is not an ancestor of $V_4$ in Fig.~\ref{fig:subdag1}, even though $X\rightarrow V_4$ is a visible edge in Fig.~\ref{fig:subpag}.

\begin{defn}[PC-Component]
In a MAG, a PAG, or any of its induced subgraphs, 
two nodes $X$ and $Y$ are in the same possible c-component (pc-component) if there is a path between the two nodes such that (1) all non-endpoint nodes along the path are 
colliders, and (2) none of the edges is visible.
\label{def:pccomp}
\end{defn}

As alluded earlier, a c-component in a causal graph plays a central role in identification.
The following proposition establishes a graphical condition in an induced subgraph $\mathcal{P}_\mathbf{A}$ that is necessary for two nodes being in the same c-component in $\mathcal{D}_\mathbf{A}$ for some DAG $\mathcal{D}$ represented by $\mathcal{P}$.

\begin{proposition}
Let $\mathcal{P}$ be a PAG over $\mathbf{V}$, and $\mathcal{D}(\mathbf{V},\mathbf{L})$ be any DAG in the equivalence class represented by $\mathcal{P}$.
For any $X, Y\in \mathbf{A}\subseteq\mathbf{V}$, 
 if $X$ and $Y$ are in the same c-component in $\mathcal{D}_\mathbf{A}$, then $X$ and $Y$ are in the same pc-component in $\mathcal{P}_\mathbf{A}$.
\label{prop:pccompsubpag}
\end{proposition}

\begin{psketch}
If $X$ and $Y$ are in the same c-component in $\mathcal{D}_\mathbf{A}$, then there is a path $p$ in $\mathcal{D}_\mathbf{A}$ composed of nodes $\langle X=V_0,\dots,Y=V_m\rangle$, $m\geq 1$, such that $V_i\in\mathbf{A}$ and $V_i\leftarrow L_{i,i+1}\rightarrow V_{i+1}$, $0\leq i<m$.
We prove that $X$ and $Y$ are in the same pc-component in $\mathcal{M}$, the MAG of $\mathcal{D}$ over $\mathbf{V}$, due to a path $p'$ over a subsequence of $p$.
We then show that $X$ and $Y$ are in the same pc-component in $\mathcal{P}$, the PAG of $\mathcal{M}$, due to a path $p^*$ over a subsequence of $p'$.
Since all the nodes along $p^*$ are in $\mathbf{A}$, then $p^*$ is present in $\mathcal{P}_\mathbf{A}$,  and so $X$ and $Y$ are in the same pc-component in $\mathcal{P}_\mathbf{A}$.
Due to space constraints, the complete proofs are provided in~\citep{appendix}.
\end{psketch}

This result provides a sufficient condition for {\em not} belonging to the same c-component in any of the relevant causal graphs. In Fig.~\ref{fig:subpag}, for example, $V_1$ and $V_4$ or $X$ and $V_4$ are not in the same pc-component, which implies by Prop.~\ref{prop:pccompsubpag} that they are not in the same c-component in $\mathcal{D}_\mathbf{A}$ for any DAG $\mathcal{D}$ in the equivalence class represented by the PAG in Fig.~\ref{fig:intropag}. 

As a special case of Def.~\ref{def:pccomp}, we define the following notion, which will prove useful later on. 

\begin{defn}[DC-Component]
In a MAG, a PAG, or any of its induced subgraphs, 
two nodes $X$ and $Y$ are in the same definite c-component (dc-component) if they are connected with a bi-directed path, i.e. a path composed solely of bi-directed edges.
\end{defn}

One challenge with the notion of \textit{pc-component} is that it is not transitive as \textit{c-component} is.
 Consider the PAG $V_1\circ\!\!\--\!\!\circ V_2\circ\!\!\--\!\!\circ V_3$.
Here, $V_1$ and $V_2$ are in the same pc-component, $V_2$ and $V_3$ are in the same pc-component, however, $V_1$ and $V_3$ are not in the same pc-component.
Hence, we define a notion that is a transitive closure of the notion of pc-component, which will prove instrumental to our goal.

\begin{defn}[CPC-Component]
Let $\mathcal{P}$ denote a PAG or a corresponding induced subgraph.
Nodes $X$ and $Y$ are in the same composite pc-component in $\mathcal{P}$, denoted \textit{cpc-component}, if there exist a sequence of nodes $\langle X=V_0,\dots,Y=V_m\rangle$, $m\geq 1$, such that $V_i$ and $V_{i+1}$ are in the same pc-component, $0\leq i<m$.
\label{def:cc-component}
\end{defn}

It follows from the above definition that 
 a PAG or an induced subgraph $\mathcal{P}$ can be decomposed into unique sets of cpc-components.
For instance, the cpc-components in Fig.~\ref{fig:subpag} are $\mathbf{S}_1=\{V_1,V_2,X\}$ and $\mathbf{S}_2=\{V_4\}$.
The significance of a \textit{cpc-component} is that it corresponds to a composite c-component in the relevant causal graphs as shown in the following proposition.

\begin{proposition}
Let $\mathcal{P}$ be a PAG over $\mathbf{V}$, $\mathcal{D}(\mathbf{V},\mathbf{L})$ be any DAG in the equivalence class represented by $\mathcal{P}$, and $\mathbf{A}\subseteq\mathbf{V}$.
If $\mathbf{C}\subseteq\mathbf{A}$ is a cpc-component in $\mathcal{P}_\mathbf{A}$, then $\mathbf{C}$ is a composite c-component in $\mathcal{D}_\mathbf{A}$.
\label{prop:cccompsubpag}
\end{proposition}

\begin{proof}
According to Definition~\ref{def:cc-component}, $\mathbf{C}$ includes all the nodes that are in the same pc-component with some node in $\mathbf{C}$ in $\mathcal{P}_\mathbf{A}$.
If follows from the contrapositive of Prop.~\ref{prop:pccompsubpag} that no node outside $\mathbf{C}$ is in the same c-component with any node in $\mathbf{C}$ in $\mathcal{D}_\mathbf{A}$.
Hence, set $\mathbf{C}$ represents a composite c-component in $\mathcal{D}_\mathbf{A}$ by Definition~\ref{def:composite-ccomp}.
\end{proof}


\begin{algorithm}[t]
 \Input{PAG $\mathcal{P}$ over $\mathbf{V}$}
 \Output{PTO over $\mathcal{P}$}
 1- Create singleton buckets $\mathbf{B_i}$ each containing $V_i\in\mathbf{V}$.\linebreak \\[-2ex]
 2- Merge buckets $\mathbf{B_i}$ and $\mathbf{B_j}$ if there is a circle edge between them ($\mathbf{B_i}\ni X\circ\!\!\--\!\!\circ Y\in \mathbf{B_j}$).\linebreak \\[-2ex]
 3- \While{set of buckets ($\mathbf{B}$) is not empty}{
  (i) Extract $\mathbf{B_i}$ with only arrowheads incident on it.\\
  (ii) Remove edges between $\mathbf{B_i}$ and other buckets.
 }\ \\[-2ex]
 4- The partial order is $\mathbf{B_1}<\mathbf{B_2}<\dots<\mathbf{B_m}$ in reverse order of the bucket extraction. Hence, $\mathbf{B_1}$ is the last bucket extracted and $\mathbf{B_m}$ is the first bucket extracted.
 \caption{PTO Algorithm}
\label{alg:pto}
\end{algorithm}

Recall that the algorithm for identification given a DAG uses a topological order over the nodes. Similarly, the algorithm we design for PAGs will depend on some 
(partial) topological order.
Thanks to the possible presence of circle edges ($\circ\!\!\--\!\!\circ$) in a PAG, in general, there may be no complete topological order that is valid for all DAGs in the equivalence class.
Algorithm~\ref{alg:pto} presents a procedure to derive a \textit{partial topological order} over the nodes in a PAG, using buckets of nodes that are connected with circle paths~\citep{jaber18}.
This algorithm remains valid over an induced subgraph of a PAG. To show this, the following lemma is crucial:

\begin{lemma}
Let $\mathcal{P}$ be a PAG over $\mathbf{V}$, and $\mathcal{P}_\mathbf{A}$ be the induced subgraph over $\mathbf{A}\subseteq\mathbf{V}$.
For any three nodes $A$, $B$, $C$, if $A*\!\!\rightarrow B\circ\!\!\--\!\!* C$, then there is an edge between $A$ and $C$ with an arrowhead at $C$, namely, $A*\!\!\rightarrow C$.
Furthermore, if the edge between $A$ and $B$ is $A\rightarrow B$, then the edge between $A$ and $C$ is either $A\rightarrow C$ or $A\circ\!\!\rightarrow C$ (i.e., it is not $A\leftrightarrow C$).
\label{lem:cpsubpag}
\end{lemma}

\begin{proof}
Lemma 3.3.1 of \citep{zhang2006causal} establishes the above property for every PAG.
 By the definition of an induced subgraph, the property is preserved in $\mathcal{P}_\mathbf{A}$.
\end{proof}

Thus, a characteristic feature of PAGs carries over to their induced subgraphs. It follows that Algorithm~\ref{alg:pto} is sound for induced subgraphs as well.

\begin{proposition}
Let $\mathcal{P}$ be a PAG over $\mathbf{V}$, and let $\mathcal{P}_\mathbf{A}$ be the subgraph of $\mathcal{P}$ induced by $\mathbf{A}\subseteq\mathbf{V}$.
Then, Algorithm~\ref{alg:pto} is sound over $\mathcal{P}_\mathbf{A}$, in the sense that the partial order is valid with respect to $\mathcal{D}_\mathbf{A}$, for every DAG $\mathcal{D}$ in the equivalence class represented by $\mathcal{P}$.
\label{prop:ptosubpag}
\end{proposition}

\begin{proof}
Let $D$ be any DAG in the equivalence class represented by $\mathcal{P}$. By Prop.~\ref{prop:ancsubpag}, the possible-ancestral relations in $\mathcal{P}_\mathbf{A}$ subsume those present in $\mathcal{D}_\mathbf{A}$.
Hence, a partial topological order that is valid with respect to $\mathcal{P}_\mathbf{A}$ is valid with respect to $\mathcal{D}_\mathbf{A}$.
The correctness of Alg.~\ref{alg:pto} with respect to a PAG in~\citep{jaber18} depends only on the property in Lemma~\ref{lem:cpsubpag}, a proof of which is given in the Supplementary Materials for completeness.
Therefore, thanks to Lemma~\ref{lem:cpsubpag}, the algorithm is also sound with respect to an induced subgraph $\mathcal{P}_\mathbf{A}$.
\end{proof}

For example, for $\mathcal{P}_\mathbf{A}$ in Fig.~\ref{fig:subpag}, a partial topological order over the nodes is $V_1<V_2<X<V_4$, which is valid for all the relevant DAGs.

With these results about induced subgraphs of a PAG, we are ready to develop a recursive approach for identification given a PAG, to which we now turn. 


\section{IDENTIFICATION IN PAGS}
\label{sec:algorithm}
\begin{figure*}[t]
\centering
\begin{subfigure}{0.66\columnwidth}
\centering
\begin{tikzpicture}
\tikzset{vertex/.style = {shape=circle,draw,minimum size=1.5em}}
\tikzset{edge/.style = {->,> = latex'}}

\node[below] (v1) at (0,0) {$V_1$};
\node[below] (v2) at (0,1.25) {$V_2$};
\node[below] (x1) at (1.25,0) {$X_1$};
\node[below] (x2) at (1.25,1.25) {$X_2$};
\node[below] (y1) at (2.75,0) {$Y_1$};
\node[below] (y2) at (4,0) {$Y_2$};
\node[below] (y3) at (4,1.25) {$Y_3$};

\draw[shorten <= 2pt,edge] (v1) to node[at start]{$\circ$} (x1);
\draw[shorten <= 2pt,edge] (v2) to node[at start]{$\circ$} (x2);
\draw[edge] (x1) to node[auto=right] {\small v} (y1);
\draw[edge] (x2) to node[auto=left] {\small v} (y3);
\draw[latex'-latex'] (x1) to (x2);
\draw[latex'-latex'] (x1) to (y3);
\draw[latex'-latex'] (x2) to (y2);
\draw[latex'-latex'] (y1) to (y2);

\end{tikzpicture}
\caption{\label{fig:idpag}}
\end{subfigure}%
\begin{subfigure}{0.66\columnwidth}
\centering
\begin{tikzpicture}
\tikzset{vertex/.style = {shape=circle,draw,minimum size=1.5em}}
\tikzset{edge/.style = {->,> = latex'}}

\node[below] (v1) at (0,0) {$V_1$};
\node[below] (v2) at (0,1.25) {$V_2$};
\node[below] (x1) at (1.25,0) {$X_1$};
\node[below] (x2) at (1.25,1.25) {$X_2$};
\node[below] (y1) at (2.75,0) {$Y_1$};
\node[below] (y2) at (4,0) {$Y_2$};

\draw[shorten <= 2pt,edge] (v1) to node[at start]{$\circ$} (x1);
\draw[shorten <= 2pt,edge] (v2) to node[at start]{$\circ$} (x2);
\draw[edge] (x1) to node[auto=right] {\small v} (y1);
\draw[latex'-latex'] (x1) to (x2);
\draw[latex'-latex'] (x2) to (y2);
\draw[latex'-latex'] (y1) to (y2);

\end{tikzpicture}
\caption{\label{fig:idsubpag1}}
\end{subfigure}%
\begin{subfigure}{0.67\columnwidth}
\centering
\begin{tikzpicture}
\tikzset{vertex/.style = {shape=circle,draw,minimum size=1.5em}}
\tikzset{edge/.style = {->,> = latex'}}

\node[below] (v1) at (0,0) {$V_1$};
\node[below] (v2) at (0,1.25) {$V_2$};
\node[below] (x1) at (1.25,0) {$X_1$};
\node[below] (y1) at (2.75,0) {$Y_1$};
\node[below] (y2) at (4,0) {$Y_2$};

\draw[shorten <= 2pt,edge] (v1) to node[at start]{$\circ$} (x1);
\draw[edge] (x1) to node[auto=right] {\small v} (y1);
\draw[latex'-latex'] (y1) to (y2);

\end{tikzpicture}
\caption{\label{fig:idsubpag2}}
\end{subfigure}
%
%
%
\caption{\label{fig:idpagex}Sample PAG $\mathcal{P}$ (left) and induced subgraphs used to identify $Q[\{Y_1,Y_2\}]$.}
\end{figure*}
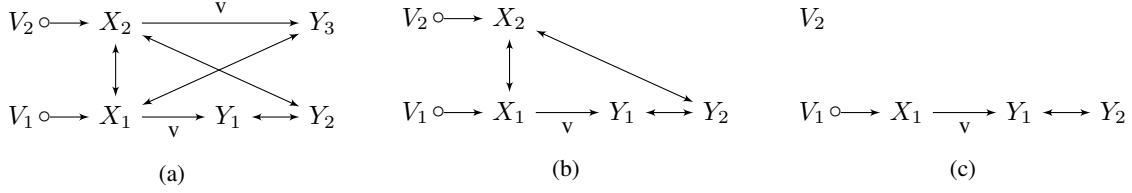

We start by formally defining the notion of identification given a PAG, which generalizes the model-specific notion \citep[pp.~70]{pearl:2k}.

\begin{defn}
Given a PAG $\mathcal{P}$ over $\mathbf{V}$ and a query $P_{\mathbf{x}}(\mathbf{y})$ where $\mathbf{X},\mathbf{Y}\subset\mathbf{V}$,
$P_{\mathbf{x}}(\mathbf{y})$ is identifiable given $\mathcal{P}$ if and only if $P_{\mathbf{x}}(\mathbf{y})$ is identifiable given every DAG $\mathcal{D}(\mathbf{V},\mathbf{L})$  in the Markov equivalence class represented by $\mathcal{P}$, and with the same expression.
\label{def:idpag}
\end{defn}

We first derive an atomic identification criterion analogous to Corollary~\ref{cor:singleiddag}.
As seen in the algorithm for constructing a partial order (Alg. ~\ref{alg:pto}), a bucket or circle component in a PAG is for our purpose analogous to a single node in a DAG. Therefore, the following criterion targets a bucket $\mathbf{X}$ rather than a single node.
\begin{theorem}
Given a PAG $\mathcal{P}$ over $\mathbf{V}$, a partial topological order $\mathbf{B}_1 <\dots<\mathbf{B}_m$ with respect to $\mathcal{P}$, a bucket 
 $\mathbf{X}\!\!=\!\!\mathbf{B}_j\!\subset\!\mathbf{T}\!\subseteq\!\mathbf{V}$, for some $1\!\!\leq\!\! j\!\!\leq\!\! m$ and where $\mathbf{T}$ is a subset of the buckets in $\mathcal{P}$,
and $P_{\mathbf{v}\setminus\mathbf{t}}$ (i.e. $Q[\mathbf{T}]$),
$Q[\mathbf{T}\setminus\mathbf{X}]$ is identifiable if and only if there does not exist $X\in \mathbf{X}$ such that $X$ has a possible child $C\notin \mathbf{X}$ that is in the same pc-component as $X$ in $\mathcal{P}_\mathbf{T}$.
If identifiable, then the expression is given by
\begin{align}
\label{eq:bucketidpag}
Q[\mathbf{T}\setminus\mathbf{X}] &= \frac{P_{\mathbf{v}\setminus\mathbf{t}}}{\prod_{\{i|\mathbf{B_i}\subseteq S^\mathbf{X}\}} P_{\mathbf{v}\setminus\mathbf{t}}(\mathbf{B_i}|\mathbf{B^{(i-1)}})}\times \\
& \qquad\qquad\sum_{\mathbf{x}} \prod_{\{i|\mathbf{B_i}\subseteq S^\mathbf{X}\}} P_{\mathbf{v}\setminus\mathbf{t}}(\mathbf{B_i}|\mathbf{B^{(i-1)}}), \nonumber
\end{align}
where $S^\mathbf{X} = \bigcup_{X\in \mathbf{X}}S^X$, $S^X$ being the dc-component of $X$ in $\mathcal{P}_\mathbf{T}$, and $\mathbf{B}^{(i-1)}$  denoting the set of nodes preceding bucket $\mathbf{B}_i$ in the partial order.
\label{th:bucketidpag}
\end{theorem}

\begin{psketch}
\textbf{(if) }
Let $\mathcal{D}$ be any DAG in the equivalence class represented by $\mathcal{P}$, $\mathcal{D}_\mathbf{T}$ be the induced subgraph over $\mathbf{T}$, and $S'$ be the smallest composite c-component containing $\mathbf{X}$ in $\mathcal{D}_\mathbf{T}$. We show that $\mathbf{X}$ is a descendant set in $\mathcal{D}_{S'}$.
Suppose otherwise for the sake of contradiction.
Then, there is a node $C\in S'\setminus\mathbf{X}$ such that $C$ is a child of $X_i$ and is in the same c-component with $X_j$, where $X_i,X_j\in\mathbf{X}$ and possibly $i=j$.
By Prop.~\ref{prop:pccompsubpag}, $X_j$ is in the same pc-component with $C$ in $\mathcal{P}_\mathbf{T}$.
Let $T_i$ be the node closest to $X_j$ along the collider path in $\mathcal{P}_\mathbf{T}$ between $X_j$ and $C$ consistent with Def.~\ref{def:pccomp}.
If the edge between $X_j$ and $T_i$ in $\mathcal{P}_\mathbf{T}$ is not into $X_j$, then $X_j$ is in the same pc-component with a possible child as the edge is not visible.
This violates the criterion stated in the theorem.
Otherwise, the edge is $X_j\leftrightarrow T_i$ and there exist a bi-directed edge between $T_i$ and every node in $\mathbf{X}$ (which follows from Lemma~\ref{lem:cpsubpag}).
Hence, $X_i$ is in the same pc-component with a possible child $C$ in $\mathcal{P}_\mathbf{T}$ (Prop.~\ref{prop:ancsubpag}), and the criterion stated in the theorem is violated again.
Therefore, $\mathbf{X}$ is a descendant set in $\mathcal{D}_{S'}$ and $Q[\mathbf{T}\setminus\mathbf{X}]$ is identifiable from $Q[\mathbf{T}]$ by Lemma~\ref{lem:subsetiddag}.
It remains to show that Eq.~\ref{eq:bucketidpag} is equivalent to Eq.~\ref{eq:subsetiddag} for $\mathcal{D}$.
The details for this step are left to the Supplementary Material.

\textbf{(only if) } 
Suppose the criterion in question is not satisfied.
Then some $X_i\in\mathbf{X}$ is in the pc-component with a possible child $C\notin\mathbf{X}$ in $\mathcal{P}_\mathbf{T}$.
The edge between $X_i$ and $C$ is $X_i*\!\!\rightarrow C$ as $C$ is outside of $\mathbf{X}$.
If the edge is not visible in $\mathcal{P}_\mathbf{T}$, then this edge is not visible in $\mathcal{P}$ (Lemma~\ref{lem:subpagvis}).
Hence, we can construct a DAG $\mathcal{D}$ in the equivalence class of $\mathcal{P}$ where $C$ is a child of $X_i$ and the two nodes share a latent variable.
The pair of sets $\mathbf{F}=\{X_i,C\}$ and $\mathbf{F}'=\{C\}$ form a so-called hedge for $Q[\mathbf{T}\setminus\mathbf{X}]$ and the effect is not identifiable in $\mathcal{D}$ \citep[Theorem 4]{shpitser2006identification}, and hence not identifiable given $\mathcal{P}$.

Otherwise, $X_i\rightarrow C$ is visible in $\mathcal{P}_\mathbf{T}$.
So, there is a collider path between $X_i$ and $C$ consistent with Def.~\ref{def:pccomp} such that the two nodes are in the same pc-component.
Let $p\!\!=\!\!\langle X_i\!\!=\!\!T_0,T_1,\dots,T_m\!\!=\!\!C\rangle$ denote the shortest such path in $\mathcal{P}_\mathbf{T}$.
If the edge between $X_i$ and $T_1$ is not into $X_i$, then $T_1$ is a child of $X_i$ and the proof follows as in the previous case.
Otherwise, we have $X_i\leftrightarrow T_1$ and we can show that $X_i$ is the only node along $p$ that belongs to $\mathbf{X}$ (details in the Supplementary Material).
In $\mathcal{P}$, path $p$ is present with $X_i\rightarrow C$ visible.
Hence, we can construct a DAG $\mathcal{D}$ in the equivalence class of $\mathcal{P}$ such that $C$ is a child of $X_i$ and both are in the same c-component through a sequence of bi-directed edges along the corresponding nodes of $p$.
The pair of sets $\mathbf{F}=\{X_i,T_1,\dots,T_m=C\}$ and $\mathbf{F}'=\{T_1,\dots,T_m=C\}$ form a hedge for $Q[\mathbf{T}\setminus\mathbf{X}]$ and the effect is not identifiable in $\mathcal{D}$, and hence it is not identifiable given $\mathcal{P}$.
\end{psketch}

Note that the above result simplifies into computing the interventional distribution $P_\mathbf{x}$ whenever the input distribution is the observational distribution, i.e. $\mathbf{T}=\mathbf{V}$.
Consider the query $P_x(\mathbf{v}\setminus\{x\})$ over the PAG in Fig.~\ref{fig:intropag}.
The intervention node $X$ is not in the same pc-component with any of its possible children ($V_3$ and $V_4$), hence the effect is identifiable and given by 
\begin{align*}
P_x(\mathbf{v}\setminus\{x\})	&= \frac{P(\mathbf{v)}}{P(x|v_1,v_2)}\times \sum_{x'} P(x'|v_1,v_2)\\
									&= P(v_1,v_2)P(v_4,v_5|v_1,v_2,x)
\end{align*}

Putting these observations together leads to the procedure we call $\mathbf{IDP}$, which is shown in Alg.~\ref{alg:idpag}.
In words, the main idea of $\mathbf{IDP}$ goes as follows.
After receiving the sets $\mathbf{X}, \mathbf{Y}$, and a PAG $\mathcal{P}$, the algorithm starts the pre-processing steps: First, it computes $\mathbf{D}$, the set of possible ancestors of $\mathbf{Y}$ in $\mathcal{P}_{\mathbf{V}\setminus\mathbf{X}}$.
Second, it uses $\mathcal{P}_\mathbf{D}$ to partition set $\mathbf{D}$ into cpc-components.
Following the pre-processing stage, the procedure calls the subroutine $\texttt{Identify}$ over each cpc-component $\mathbf{D}_i$ to compute $Q[\mathbf{D}_i]$ from the  observational distribution $P(\mathbf{V})$.
The recursive routine basically checks for the presence of a bucket $\mathbf{B}$ in $\mathcal{P}_\mathbf{T}$ that is a subset of the intervention nodes, i.e. $\mathbf{B}\subseteq\mathbf{T}\setminus\mathbf{C}$, and satisfies the conditions of Thm.~\ref{th:bucketidpag}.
If found, it is able to successfully compute $Q[\mathbf{T}\setminus\mathbf{B}]$ using Eq.~\ref{eq:bucketidpag}, and proceed with a recursive call.
Alternatively, if such a bucket doesn't exist in $\mathcal{P}_\mathbf{T}$, then $\mathbf{IDP}$ throws a failure condition, since it's unable to identify the query.
We show next that this procedure is, indeed, correct. 

\begin{theorem}
Algorithm $\mathbf{IDP}$ (Alg.\ref{alg:idpag}) is sound.
\end{theorem}

\begin{proof}
Let $\mathcal{G}(\mathbf{V},\mathbf{L})$ be any causal graph in the equivalence class of PAG $\mathcal{P}$ over $\mathbf{V}$, and let $\mathbf{V}'=\mathbf{V}\setminus\mathbf{X}$.
We have
\begin{align*}
P_\mathbf{x}(\mathbf{y}) = \sum_{\mathbf{v}'\setminus\mathbf{y}}P_\mathbf{x}(\mathbf{v}') = \sum_{\mathbf{v}'\setminus\mathbf{y}} Q[\mathbf{V}'] = \sum_{\mathbf{v}'\setminus\mathbf{d}} \sum_{\mathbf{d}\setminus\mathbf{y}}Q[\mathbf{V}']
\end{align*}

By definition, $\mathbf{D}$ is an ancestral set in $\mathcal{P}_{\mathbf{V}'}$, and hence it is ancestral in $\mathcal{G}_{\mathbf{V}'}$ by Prop.~\ref{prop:ancsubpag}.
So, we have the following by \citep[Lemma 10]{tian2002studies}:
\begin{align}
P_\mathbf{x}(\mathbf{y}) = \sum_{\mathbf{d}\setminus\mathbf{y}} \sum_{\mathbf{v}'\setminus\mathbf{d}} Q[\mathbf{V}'] = \sum_{\mathbf{d}\setminus\mathbf{y}}  Q[\mathbf{D}] 
\label{eq:idsound1}
\end{align}

Using Prop.~\ref{prop:cccompsubpag}, each cpc-component in $\mathcal{P}_\mathbf{D}$ corresponds to a composite c-component in $\mathcal{G}_\mathbf{D}$.
Hence, Eq.~\ref{eq:idsound1} can be decomposed as follows by \citep[Lemma 11]{tian2002studies}.
\begin{align}
P_\mathbf{x}(\mathbf{y}) = \sum_{\mathbf{d}\setminus\mathbf{y}}  
 Q[\mathbf{D}] = \sum_{\mathbf{d}\setminus\mathbf{y}}\prod_i Q[\mathbf{D}_{i}]
\label{eq:idsound2}
\end{align}

\begin{algorithm}[t]
\Input{two disjoint sets $\mathbf{X},\mathbf{Y}\subset\mathbf{V}$}
\Output{Expression for $P_\mathbf{x}(\mathbf{y})$ or \texttt{FAIL}}
\begin{enumerate}
\item Let $\mathbf{D}=\texttt{An}(\mathbf{Y})_{\mathcal{P}_{\mathbf{V}\setminus\mathbf{X}}}$
\item Let the cpc-components of $\mathcal{P}_{\mathbf{D}}$ be $\mathbf{D}_{i}$, $i=1,\dots,k$
\item $P_\mathbf{x}(\mathbf{y})=\sum_{\mathbf{d}\setminus\mathbf{y}}\prod_i \texttt{Identify}(\mathbf{D}_{i},\mathbf{V},P)$
\end{enumerate}
\SetKwFunction{FIdentify}{Identify}
  \SetKwProg{Fn}{Function}{:}{}
  \Fn{\FIdentify{$\mathbf{C}$, $\mathbf{T}$, $Q=Q[\mathbf{T}]$}}{
\If{$\mathbf{C}=\mathbf{T}$}{
    \Return $Q[\mathbf{T}]$\;
}

\relsize{-1}{\tcc{In $\mathcal{P}_\mathbf{T}$, let $\mathbf{B}$ be a bucket, and $C^\mathbf{B}$ be the pc-component of $\mathbf{B}$}}\normalsize
\uIf{$\exists \mathbf{B}\subseteq\mathbf{T}\setminus\mathbf{C}$ such that $C^\mathbf{B}\cap\texttt{Ch}(\mathbf{B})\subseteq\mathbf{B}$}{
	Compute $Q[\mathbf{T}\setminus\mathbf{B}]$ from $Q$; \relsize{-1}{\tcp{Theorem~\ref{th:bucketidpag}}}\normalsize
	\Return $\texttt{Identify}(\mathbf{C},\mathbf{T}\setminus\mathbf{B},Q[\mathbf{T}\setminus\mathbf{B}])$\;
}
\Else{
	\Throw \texttt{FAIL}\;
}}
\caption{\label{alg:idpag}$\mathbf{IDP}(\mathbf{x},\mathbf{y})$ given PAG $\mathcal{P}$}
\end{algorithm}
Eq.~\ref{eq:idsound2} is equivalent to the decomposition we have in step 3 of Alg.~\ref{alg:idpag}, where we attempt to compute each $Q[\mathbf{D}_{i}]$ from $P$.
Finally, the correctness of the recursive routine $\texttt{Identify}$ follows from that of Theorem~\ref{th:bucketidpag}.
\end{proof}

\subsection{ILLUSTRATIVE EXAMPLE}
\label{subsec:idpex}
Consider the query $P_{x_1,x_2}(y_1,y_2,y_3)$ given $\mathcal{P}$ in Fig.~\ref{fig:idpag}.
We have $\mathbf{D}=\{Y_1,Y_2,Y_3\}$, and the cpc-components in $\mathcal{P}_\mathbf{D}$ are $\mathbf{D}_1=\{Y_1,Y_2\}$ and $\mathbf{D}_2=\{Y_3\}$.
Hence, the problem reduces to computing $Q[\{Y_1,Y_2\}]\cdot Q[\{Y_3\}]$.

We start with the call $\texttt{Identify}(\mathbf{D}_1,\mathbf{V},P)$.
Consider the singleton bucket $Y_3$ the pc-component of which includes all the nodes in $\mathcal{P}$.
This node satisfies the condition in $\texttt{Identify}$ as it has no children, and we compute $Q[\mathbf{V}\setminus\{Y_3\}]$ using Theorem~\ref{th:bucketidpag}.
\begin{align}
Q[\mathbf{V}\setminus\{Y_3\}]	&= \frac{P(\mathbf{v})}{P(y_1,y_2,y_3,x_1,x_2|v_1,v_2)} \times\nonumber\\
														&	 \qquad\sum_{y_3} P(y_1,y_2,y_3,x_1,x_2|v_1,v_2)\nonumber\\
														&= P(v_1,v_2)\cdot P(y_1,y_2,x_1,x_2|v_1,v_2)\nonumber\\
                                                        &= P(y_1,y_2,x_1,x_2,v_1,v_2)
\label{eq:idex1}
\end{align}

In the next recursive call, $\mathbf{T}_1=\mathbf{V}\setminus\{Y_3\}$,  $P_{y_3}$  corresponds to Eq.~\ref{eq:idex1}, and the induced subgraph $\mathcal{P}_{\mathbf{T}_1}$ is shown in Fig.~\ref{fig:idsubpag1}.
Now, $X_2$ satisfies the criterion and we can compute $Q[\mathbf{T}_1\setminus\{X_2\}]$ from $P_{y_3}=Q[\mathbf{T}_1]$, i.e., 
\begin{align}
Q[\mathbf{T}_1\setminus\{X_2\}]	&= \frac{P_{y_3}}{P_{y_3}(y_1,y_2,x_1,x_2|v_1,v_2)} \times\nonumber\\
															&	\qquad \sum_{x_2} P_{y_3}(y_1,y_2,x_1,x_2|v_1,v_2)\nonumber\\
                                                            &= P(y_1,y_2,x_1,v_1,v_2)
\label{eq:idex2}
\end{align}

Let $\mathbf{T}_2=\mathbf{T}_1\setminus\{X_2\}$, where the induced subgraph $\mathcal{P}_{\mathbf{T}_2}$ is shown in Fig.~\ref{fig:idsubpag2}.
Now, $X_1$ satisfies the criterion and we can compute $Q[\mathbf{T}_2\setminus\{X_1\}]$ from Eq.~\ref{eq:idex2}, 
\begin{align*}
Q[\mathbf{T}_2\setminus\{X_1\}]	&= \frac{P_{y_3,x_2}}{P_{y_3,x_2}(x_1|v_1,v_2)} \times\\
															&	\qquad \sum_{x_1} P_{y_3,x_2}(x_1|v_1,v_2)\\
															&= \frac{P(v_1,v_2)\cdot P(y_1,y_2,x_1,v_1,v_2)}{P(x_1,v_1,v_2)}\\
															&= P(v_1,v_2)\cdot P(y_1,y_2|x_1,v_1,v_2)
\end{align*}

Choosing $V_1$ and $V_2$ in the next two recursive calls, we finally obtain the simplified expression:
\begin{align*}
Q[\{Y_1,Y_2\}] = P(y_1,y_2|x_1)
\label{eq:idex3}
\end{align*}

Next, we solve for $Q[\mathbf{D}_2]$ and we get an expression analogous to that of $Q[\mathbf{D}_1]$.
Hence, the final solution is:
\begin{align*}
P_{x_1,x_2}(y_1,y_2,y_3) = P(y_1,y_2|x_1)\times P(y_3|x_2)
\end{align*}

\subsection{COMPARISON TO STATE OF THE ART} 
\label{sec:comparison}
In the previous section, we formulated an identification algorithm in PAGs for causal queries of the form $P_\mathbf{x}(\mathbf{y})$, $\mathbf{X},\mathbf{Y}\subset\mathbf{V}$.
A natural question arises about the expressiveness of the $\mathbf{IDP}$  in comparison with the state-of-the-art methods.
One of the well established results in the literature is the adjustment method  \citep{perkovic2015complete}, which is complete whenever an adjustment set exists.

In the sequel, we formally show that the proposed algorithm subsumes the adjustment method.
\begin{figure}[t]
\centering
\begin{tikzpicture}[scale=0.9]
\tikzset{vertex/.style = {shape=circle,draw,minimum size=1.5em}}
\tikzset{edge/.style = {->,> = latex'}}

\node[below] (v1) at (0,0) {$V_1$};
\node[below] (x) at (1.5,0) {$X$};
\node[below] (v2) at (1.5,3) {$V_2$};
\node[below] (v3) at (2.25,1.5) {$V_3$};
\node[below] (v4) at (3.75,1) {$V_4$};
\node[below] (z) at (5,0) {$Z$};
\node[below] (y) at (7,0) {$Y$};

\draw[shorten <= 2pt,edge] (v1) to node[at start]{$\circ$} (x);
\draw[latex'-latex'] (x) to (v2);
\draw[latex'-latex'] (v2) to (v3);
\draw[latex'-latex'] (v3) to (v4);
\draw[latex'-latex'] (v4) to (z);
\draw[edge] (v3) to node[auto=left] {\small v} (x);
\draw[edge] (x) to node[auto=left] {\small v} (z);
\draw[edge] (v2) to node[auto=left] {\small v} (y);
\draw[edge] (v4) to node[auto=left] {\small v} (y);
\draw[edge] (z) to node[auto=left] {\small v} (y);

\end{tikzpicture}
\caption{\label{fig:idsingle}Query $P_x(y)$ is identifiable by $\mathbf{IDP}$.}
\vspace{-.1in}
\end{figure}
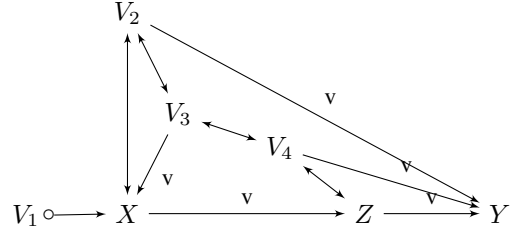

\begin{theorem}
Let $\mathcal{P}$ be a PAG over set $\mathbf{V}$ and let $P_{\mathbf{x}}(\mathbf{y})$ be a causal query where $\mathbf{X},\mathbf{Y}\subset\mathbf{V}$.
If the distribution $P_{\mathbf{x}}(\mathbf{y})$ is not identifiable using $\mathbf{IDP}$ (Alg.~\ref{alg:idpag}), then the effect is not identifiable using the \textit{generalized adjustment criterion} in~\citep{perkovic2015complete}.
\label{th:idvsgac}
\end{theorem}
\vspace{-1em}
\begin{psketch}
Whenever $\mathbf{IDP}$ fails to identify some query, it is due to one of the recursive calls to $\texttt{Identify}$.
We use the failing condition inside this call to systematically identify a proper definite status non-causal path from $\mathbf{X}$ to $\mathbf{Y}$ in $\mathcal{P}$ that is m-connecting given set Adjust($\mathbf{X}$,$\mathbf{Y}$,$\mathcal{P}$) \citep[Def. 4.1]{perkovic2016complete}.
As this set fails to satisfy the adjustment criterion, then there exist no adjustment set relative to the pair ($\mathbf{X}$,$\mathbf{Y}$) in $\mathcal{P}$ \citep[Cor. 4.4]{perkovic2016complete}.
The details of the proof are left to the Supplementary Material.
\end{psketch}
Based on this result, one may wonder whether these algorithms are, after all, just equivalent. In reality, \textbf{IDP} captures strictly more identifiable effects than the adjustment criterion. To witness, consider  the PAG in
 Fig.~\ref{fig:idsingle} and note that the causal distribution $P_x(y)$ is identifiable by $\mathbf{IDP}$ but not 
by adjustment in this case.



\section{CONCLUSION}
\label{sec:conclusion}
\vspace{-0.10in}
We studied the problem of identification of interventional distributions in Markov equivalence classes represented by PAGs.
We first investigated graphical properties for induced subgraphs of PAGs over an arbitrary subset of nodes with respect to induced subgraphs of DAGs that are in the equivalence class.
We believe that these results can be useful to general tasks related to causal inference from equivalence classes. 
We further developed an identification algorithm in PAGs and proved it to subsume the state-of-the-art adjustment method.

\section*{Acknowledgments}
\vspace{-0.1in}
We thank Sanghack Lee and the reviewers for all the feedback provided. Bareinboim and Jaber are supported in parts by grants from NSF IIS-1704352 and IIS-1750807 (CAREER). Zhang is supported in part by the Research Grants Council of Hong Kong under the General Research Fund LU13600715.

\bibliographystyle{plainnat}
\bibliography{main}

\end{document}